\pdfoutput=1                         
\documentclass{article}

\usepackage[margin=1.15in,footskip=30pt]{geometry}

\usepackage{mathpazo}

\usepackage[round,sort,comma]{natbib}

\usepackage[utf8]{inputenc}         
\usepackage[T1]{fontenc}            
\usepackage{microtype}              

\usepackage{amsmath,amssymb,amsfonts,amsthm}
\usepackage{stmaryrd}
\usepackage{mathrsfs}

\usepackage{graphicx}
\usepackage{subcaption}
\usepackage{wrapfig}

\usepackage{algorithm}
\usepackage{algpseudocode}          
\algrenewcommand\algorithmicindent{1.2em} 
\usepackage{booktabs}
\usepackage{tabularx}
\usepackage{array}
\usepackage{multirow}
\usepackage{makecell}
\usepackage{siunitx}                
\usepackage{arydshln}               
\usepackage{colortbl}               
\usepackage{balance}

\usepackage{xcolor}
\usepackage{pifont}                 
\definecolor{mydarkblue}{rgb}{0,0.08,0.45}

\usepackage[colorlinks,allcolors=mydarkblue]{hyperref}
\usepackage{nicefrac}
\usepackage{xr}                     
\usepackage{soul}                   
\usepackage{blindtext}              

\usepackage{thmtools}
\usepackage{thm-restate}

\usepackage{tikz}
\usetikzlibrary{arrows.meta,positioning,automata}
\usepackage{subcaption}
\usepackage{graphicx} 

\tikzset{
  >=Latex, initial text=, 
  every state/.style={minimum size=6.5mm, inner sep=0pt, line width=.6pt},
  every loop/.style={looseness=6},
  orig/.style={state, fill=green!18},
  added/.style={state, fill=blue!18},
  dead/.style={state, draw=black!70, fill=black!7}
}

\definecolor{burntorange}{rgb}{0.8,0.33,0.0}

\theoremstyle{definition}

\theoremstyle{remark}

\theoremstyle{plain}
\newtheorem{theorem}{Theorem}[section]
\newtheorem{corollary}{Corollary}[theorem]

\newtheorem{lemma}[theorem]{Lemma}

\definecolor{xdarkblue}{rgb}{0,0.08,0.45}
\definecolor{sbblue}{HTML}{023a75}
\definecolor{tdred}{HTML}{cf0202}

\usepackage{amsmath,amsfonts,bm}

\def\eqref#1{(\ref{#1})}

\def\1{\bm{1}}

\DeclareMathAlphabet{\mathsfit}{\encodingdefault}{\sfdefault}{m}{sl}
\SetMathAlphabet{\mathsfit}{bold}{\encodingdefault}{\sfdefault}{bx}{n}

\newcommand{\E}{\mathop{\mathbb{E}}}

\usepackage{commath}
\usepackage{amssymb}
\usepackage{amsmath,amsfonts,bm}

\newcommand{\ham}{\{0,1\}}

\newcommand{\DFA}{\mathrm{DFA}}
\newcommand{\DFAn}{\mathrm{DFA}_n}
\newcommand{\BDFA}{\mathrm{ADFA}}

\newcommand{\A}{\mathcal{A}}

\newcommand{\Astar}{A^{\star}}
\newcommand{\fstar}{f_{\Astar}}

\newcommand{\fhat}{\hat{f}}

\newcommand{\nsp}{\mathrm{NSP}}

\newcommand{\deadstate}{q_\mathrm{dead}}
\newcommand{\dfunc}{\varphi}

\newcommand{\alphasize}{|\Sigma|}

\newcommand{\calD}{\mathcal{D}}

\newcommand{\calF}{\mathcal{F}}

\newcommand{\calA}{\mathcal{A}}

\newcommand{\err}{\mathrm{err}}

\newcommand{\errnsp}{\mathcal{L}_{\mathrm{NSP}}}

\newcommand{\pref}{x_{:n}}
\newcommand{\prefl}[1]{x_{:{#1}}}

\newcommand{\sym}{\sigma}
\newcommand{\final}{F}

\title{Hardness of Learning Regular Languages in \\ the Next Symbol Prediction Setting }

\author{Satwik Bhattamishra$^1$\thanks{Corresponding author. Contact: \href{mailto:satwik.bmishra@cs.ox.ac.uk}{\texttt{\footnotesize satwik.bmishra@cs.ox.ac.uk}}} \qquad Phil Blunsom$^{2}$ \qquad Varun Kanade$^1$ \\
\vspace{-2mm} \\
  \normalsize{$^1$University of Oxford\quad  \quad $^2$Cohere} \\
    } 
\date{}

\begin{document}

\maketitle

\begin{abstract}
We study the learnability of languages in the \emph{Next Symbol Prediction} (NSP) setting, where a learner receives only positive examples from a language together with, for every prefix, (i) whether the prefix itself is in the language and (ii) which next symbols can lead to an accepting string. This setting has been used in prior works to empirically analyze neural sequence models, and additionally, we observe that efficient algorithms for the NSP setting can be used to learn the (truncated) support of language models. We formalize the setting so as to make it amenable to PAC-learning analysis. While the setting provides a much richer set of labels than the conventional classification setting, we show that learning concept classes such as DFAs and Boolean formulas remains computationally hard. The proof is via a 
construction that makes almost all additional labels uninformative, yielding a reduction from the conventional learning problem to learning with NSP labels. Under cryptographic assumptions, the reduction implies that the problem of learning DFAs is computationally hard in the NSP setting.

\end{abstract}

\section{Introduction}\label{sec:intro}

We formalize and study the problem of learnability of languages in the \textit{Next Symbol Prediction} ($\nsp$) setting, in which a learner receives only \emph{positive} strings from a target language, together with rich supervision for every prefix: a membership bit indicating whether the prefix itself is in the language and a vector of ``continuation'' bits indicating which next symbols admit some accepting continuation. A prediction is correct only if the hypothesis matches \emph{all} membership and continuation labels at \emph{every} prefix of the example.

This setup has a natural interpretation in the context of language models: when decoding with top-$p$~\citep{holtzman2019curious}, top-$k$, or min-$p$~\citep{minh2025turning} sampling, the per-prefix continuation set is precisely the set of admissible next tokens, and termination corresponds to allowing the end-of-sequence token. In particular, positive-only $\nsp$ supervision is naturally obtained from black-box models and avoids the need for inventing artificial distributions over negative strings; at the same time, the requirement for correct predictions at every prefix makes the task challenging. 

Additionally, while $\nsp$ has been widely used to \emph{evaluate} sequence models on formal-language benchmarks (see e.g. \citet{gers2001lstm,suzgun2019lstm,ebrahimi2020can,bhattamishra-etal-2020-ability} and references therein), the learnability of languages under $\nsp$ and its relationship to conventional binary classification has not been established. 

\textbf{Our contributions.} We give a PAC-style formulation of learning using $\nsp$ labels and analyze the learnability of classes such as DFAs and Boolean formulas. We prove that $\nsp$ supervision does \emph{not} remove the key \emph{computational barriers} for learning regular languages. The key technical argument is a construction that renders all but one continuation label uninformative, allowing a reduction to well-known hardness results \citep{kearns1994cryptographic}. Thus, even with the richer labels, efficient (improper) learning remains computationally intractable (under standard cryptographic assumptions). Our results suggest that while $\nsp$ labels may offer some benefit, they do not, in general, circumvent computational hardness.\footnote{In an \href{https://openreview.net/forum?id=3vkeYFEZxW}{earlier workshop paper} \citep{bhattamishra2023the}, we argued that polynomial-time learning is feasible in the NSP setting. We later found that the claim was \textit{incorrect}; in this work, we prove that learning in the NSP setting is computationally hard.}

\section{Problem Definition}\label{sec:def}

\textbf{Notation.}
A deterministic finite automaton ($\DFA$) is a tuple $A=(Q,\Sigma,\delta,q_0,F_A)$ with finite state set $Q$, alphabet $\Sigma$, transition function $\delta:Q\times\Sigma\to Q$, start state $q_0$, and a subset of accepting states $F_A\subseteq Q$. The language of $A$ is $L_A\subseteq\Sigma^{*}$; write $A(x)=L_A(x)\in\{0,1\}$ and $|A|$ denotes the number of states $|Q|$. Fix an order $\Sigma=\{\sigma_1,\ldots,\sigma_{\alphasize}\}$. For $x=w_1\cdots w_N$, let the length-$n$ prefix be $\pref:=w_1\cdots w_n$ for $0\le n\le N$. We use $\deadstate$ for a dead state with $\delta(\deadstate,\sigma)=\deadstate$ for all $\sigma\in\Sigma$; if present, such a state is unique in a minimal DFA. We use $\DFAn$ to denote the class of DFAs with at most $n$ states.

\subsection{Next Symbol Prediction (NSP) Setting}\label{subsec:nspdef}


For a language $L$, for any string $x$ and any $\sigma \in \Sigma$, we define $\dfunc_L(x, \sigma)$ as follows: 

\[
\dfunc_L(x,\sigma):=
\begin{cases}
1, & \text{if }\,\,\exists \, s \in \Sigma^* \,\text{ such that } \, x\cdot\sigma\cdot s \in L\\
0, & \text{otherwise}.
\end{cases}
\]

When the language $L$ is regular, then it is equivalent to computing whether $x \cdot \sigma$ leads to a dead state in its canonical DFA. For any minimal DFA, let $q= \delta(q_0, x)$, then $\dfunc(x, \sigma) = \dfunc(q, \sigma) = 0$ if $\delta(q, \sigma) = \deadstate$ and $1$ otherwise. The continuation vector at $q$ is defined as $\dfunc(q)=[\dfunc(q,\sigma_1),\ldots,\dfunc(q,\sigma_{|\Sigma|})]\in\{0,1\}^{|\Sigma|}$, and for $x\in\Sigma^{*}$ we use $\dfunc(x):=\dfunc(\delta(q_0,x))$ and $\dfunc(x,\sigma):=\dfunc(\delta(q_0,x),\sigma)$. 


Fix a target language $L\subseteq\Sigma^{*}$. An example consists of a positive string $x=w_1\cdots w_N\in L$ together with labels for each prefix $\pref$ ($0\le n\le N$): its membership label $L(\pref)\in\{0,1\}$ and the continuation labels $(\dfunc(\pref,\sigma))_{\sigma\in\Sigma}\in\{0,1\}^{|\Sigma|}$, where $\dfunc(\pref,\sigma)=1$ means that $\pref\cdot\sigma$ has some accepting continuation in $L$ and $\dfunc(\pref,\sigma)=0$ implies that $\pref \cdot \sigma \cdot s \notin L$ for every suffix $s$. Collecting these across all prefixes gives
\[
f_{L}(x)\ :=\ \big( \, (\dfunc(\pref,\sigma_i))_{i=1}^{|\Sigma|}, \, L(\pref) \, \big)_{n=0}^{N}
\ \in\ \{0,1\}^{(|\Sigma|+1)(N+1)},
\]
with coordinates ordered by nondecreasing prefix length $n$.

\textbf{Predictors and loss.}
Although the $\nsp$ setting is not restricted to regular languages, it is more intuitive to initially think of the setting in terms of DFAs. We instantiate predictors via automata. For a DFA $A$, define
\[
f_A(x)\ :=\ \big( \, (\dfunc(\pref,\sigma_i))_{i=1}^{|\Sigma|}, \, L_A(\pref) \, \big)_{n=0}^{|x|}, \qquad
\forall x \in \Sigma^*, \quad f_A(x) \in \{0,1\}^{(|\Sigma|+1)(|x|+1)} .
\]
Let $\Astar$ denote an unknown target DFA. On a single example $x$, the error is defined as,
\[
\err\big(f_A(x),\fstar(x)\big)\ :=\ \mathbb{I} [f_A(x) \neq \fstar(x)].
\]
Note that for a given input string $x$, the error $\err(f_A(x),\fstar(x))=0$ precisely when all $(|\Sigma|+1)(|x|+1)$ labels are correct. For a distribution $\calD$ supported on positive examples $L$, the expected $\nsp$ error is
\[
\errnsp(f_A;\fstar,\calD)\ :=\ \E_{x\sim \calD}\big[\, \err\big(f_A(x),\fstar(x)\big) \,\big].
\]

\textbf{Remark.}  Note that, although the $\nsp$ setting only involves positive examples, some information about the negative examples can be obtained from the membership and continuation labels of the prefixes. Since a predictor needs to predict all $(|\Sigma|+1)(|x|+1)$ correctly for an example $x$, the setting is more stringent in the sense that a random predictor will typically have near-zero accuracy as opposed to 50\% accuracy in the conventional classification setting. We will refer to labels described above of the form $f_A(x) \in \{0, 1\}^{(|\Sigma|+1)(|x|+1)}$ as the $\nsp$ labels. We refer to the setting with binary classifiers $f(x) \in \ham$ as the conventional classification setting.

\section{Hardness of Learning}\label{sec:nsphard}

 We study efficient PAC learnability in the $\nsp$ setting. An algorithm $\calA$ is an efficient PAC learner for a class $\calF$ if for every $f\in\calF$ and every distribution $\calD$ supported on positive examples, $\calA$ runs in polynomial time on $\nsp$–labeled inputs and outputs $\hat f$ such that, with probability at least $1-\delta$, $\errnsp(\hat f;f,\calD)\le\epsilon$.

 Note that the continuation labels can be highly informative for certain classes. Consider \emph{Conjunctions}: Boolean monomials $f:\{0,1\}^N\to\{0,1\}$, e.g., $f(z)=z_2\wedge \bar z_4$. In the conventional classification model, learning Conjunctions is known to require $\Theta(N)$ labeled examples. In the $\nsp$ model, one positive example $x\in f^{-1}(1)$ suffices. For each prefix $x_{:k}$, the label $\dfunc(x_{:k},0)$ equals $0$ if and only if the literal $z_{k+1}$ appears in the target conjunction; likewise, $\dfunc(x_{:k},1)=0$ if and only if the literal $\bar z_{k+1}$ appears. Thus, by reading the continuation labels across the $N$ prefixes, the learner recovers exactly which literals are present and hence identifies the target monomial from a single positive $\nsp$ example. While the $\nsp$ labels may provide a statistical advantage for some classes, we show that in the general case, they are not enough to remove the computational barriers for learning DFAs.

We relate learnability in the $\nsp$ setting to standard PAC learning for acyclic DFAs that accept only strings of a fixed length.  Throughout this section, we work over the binary alphabet $\Sigma=\{0,1\}$.

\textbf{Notation for classes.}
For $N\in\mathbb{N}$ and a polynomial $p(\cdot)$, define
\[
\mathrm{ADFA}^{N}_{p(N)}\ :=\ \big\{A:\ A \text{ is a DFA with at most } p(N) \text{ states and } L_A\subseteq\{0,1\}^{N}\big\},
\]
and write $\mathrm{ADFA}_{p(\cdot)}:=\bigcup_{N\ge 1}\mathrm{ADFA}^{N}_{p(N)}$.  
As above, $\DFAn$ denotes the class of DFAs with at most $n$ states. Under the convention that the transition function is total, any minimal $A\in\mathrm{ADFA}^{N}_{p(N)}$ has a dead (sink) state with self-loops on both input symbols. If instead one allows a partial transition function $\delta$, interpreting undefined transitions as implicit moves to a dead state, then the automata in $\mathrm{ADFA}^{N}_{p(N)}$ are acyclic in the usual sense.

\textbf{Background.}
\citet{kearns1994cryptographic} show that, for a suitable polynomial $p$, weak PAC learning of $\mathrm{ADFA}_{p(\cdot)}$ in the conventional classification (binary-label) model is as hard as inverting certain cryptographic functions.

\begin{theorem}[\citet{kearns1994cryptographic}]\label{th:kearns}
There exists a polynomial $p(\cdot)$ such that the problems of inverting RSA, factoring Blum integers, etc., are probabilistic polynomial-time reducible to weakly learning $\mathrm{ADFA}_{p(\cdot)}$ in the standard PAC setting.
\end{theorem}

We prove that an efficient PAC algorithm for $\mathrm{ADFA}^{N}_{p(N)}$ in the $\nsp$ setting would yield an efficient PAC algorithm for $\mathrm{ADFA}^{N}_{p(N)}$ in the conventional classification setting.  Together with Theorem~\ref{th:kearns}, this implies cryptographic hardness for $\nsp$ learning of these Boolean Acyclic DFAs and consequently the general class of DFAs.

\subsection{Boolean Acyclic DFAs}\label{subsec:bdfa-construct}

Fix $N\ge 1$ and a target DFA $A=(Q,\Sigma,\delta,q_0,\final)$ with $L_A\subseteq\{0,1\}^{N}$.  We first record a basic structural property of \emph{minimal} DFAs for fixed-length languages.

\begin{lemma}[Unique depth in minimal \texorpdfstring{$\mathrm{ADFA}^{N}_{p(N)}$}{ADFA}]\label{lem:depth}
If $A$ is minimal and $L_A\subseteq\{0,1\}^{N}$, then every state $q\in Q\setminus\{\deadstate\}$ is reachable by strings of exactly one length $\ell(q)\in\{0,1,\ldots,N\}$.  Consequently, every transition increases depth by one: if $\delta(q,\sigma)=q'$ with $q\neq \deadstate$ and $q'\neq \deadstate$, then $\ell(q')=\ell(q)+1$.  Acceptance occurs only at depth $N$.
\end{lemma}

\begin{proof}
Let $q\in Q\setminus\{\deadstate\}$ be reachable both by a string of length $\ell$ and by a string of length $\ell'$ with $\ell<\ell'$.  The residual language at $q$ is $R(q)\ :=\ \{\,s\in \Sigma^{*}: \delta(q,s)\in F\,\}$.
If $q$ is reached after $\ell$ symbols, then every $s\in R(q)$ must have length exactly $N-\ell$; if $q$ is reached after $\ell'$ symbols, then every $s\in R(q)$ must have length exactly $N-\ell'$.  Since $N-\ell\neq N-\ell'$, these sets are disjoint; hence $R(q)$ must be empty, contradicting $q\neq\deadstate$ in a minimal DFA.  Thus, each non-dead state has a unique depth $\ell(q)$.  Any transition consumes one symbol, so $\ell(q')\le \ell(q)+1$; equality must hold by uniqueness of depth.  Finally, if $q\in F$ had $\ell(q)\neq N$, then $L_A$ would contain strings of length other than $N$, contrary to the assumption.
\end{proof}

By Lemma~\ref{lem:depth}, we may index non-dead states by their unique depth $\ell(q)\in\{0,\ldots,N\}$ (the dead state has no depth).

\textbf{Padding by one bit.}
We construct from $A$ a DFA $A^{\oplus}$ over length-$N{+}1$ inputs whose $\nsp$ labels are uninformative before depth $N$, while at depth $N$ the continuation bit for symbol $0$ recovers $A$’s label. 

More precisely, we will construct a DFA $A^{\oplus}$ with at most $|A| + N + 1$ states such that: (i) the $\nsp$ labels of any prefix of length $<N$ will be $(1, 1, 0)$, and (ii) crucially, the $\nsp$ label for a string $u \in \Sigma^N$ of length $N$ will be $(A(u), 1, 0)$. The ADFA $A^{\oplus}$ accepts all strings $x$ of length $N+1$ ending with $1$ and some ending with $0$ depending on $A(x_{:N})$. Thus, with respect to $A^{\oplus}$, predicting the first continuation bit (corresponding to symbol $0$) is equivalent to predicting membership in $L_A$.

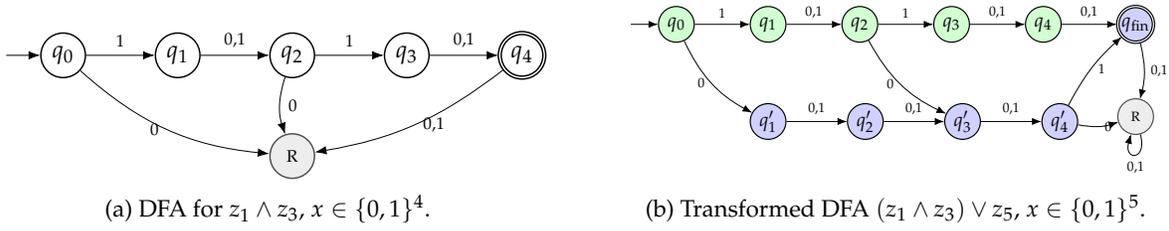
\begin{figure}[t]
\centering
\captionsetup[subfigure]{justification=centering}

\begin{subfigure}[t]{0.47\linewidth}
\centering
\resizebox{\linewidth}{!}{%
\begin{tikzpicture}[node distance=10mm]
  \node[state, initial] (s0) {$q_0$};
  \node[state, right=of s0] (s1) {$q_1$};
  \node[state, right=of s1] (s2) {$q_2$};
  \node[state, right=of s2] (s3) {$q_3$};
  \node[state, accepting, right=of s3] (sf) {$q_4$};
  \node[dead, below=8mm of s2] (sd) {\scriptsize R};

  \path[->]
    (s0) edge node[above, font=\scriptsize] {1} (s1)
    (s1) edge node[above, font=\scriptsize] {0,1} (s2)
    (s2) edge node[above, font=\scriptsize] {1} (s3)
    (s3) edge node[above, font=\scriptsize] {0,1} (sf);

  \path[->]
    (s0) edge[bend right=18] node[left, font=\scriptsize] {0} (sd)
    (s2) edge[bend right=18] node[right, font=\scriptsize] {0} (sd)
    (sf) edge[bend left=14] node[right, font=\scriptsize] {0,1} (sd);
    (sd) edge[loop below] node[font=\scriptsize] {0,1} (sd);
\end{tikzpicture}%
}
\caption{DFA for $z_1\wedge z_3$, $x\in\{0,1\}^4$.}
\end{subfigure}\hfill
\begin{subfigure}[t]{0.47\linewidth}
\centering
\resizebox{\linewidth}{!}{%
\begin{tikzpicture}[node distance=10mm]
  \node[orig, state, initial] (t0) {$q_0$};
  \node[orig, state, right=of t0] (t1) {$q_1$};
  \node[orig, state, right=of t1] (t2) {$q_2$};
  \node[orig, state, right=of t2] (t3) {$q_3$};
  \node[orig, state, right=of t3] (t4) {$q_4$};
  \node[added, state, accepting, right=of t4] (tf) {$q_{\text{fin}}$};

  \node[added, state, below=11mm of t1] (b1) {$q'_1$};
  \node[added, state, right=11mm of b1] (b2) {$q'_2$};
  \node[added, state, right=11mm of b2] (b3) {$q'_3$};
  \node[added, state, right=11mm of b3] (b4) {$q'_4$};
  \node[dead, below=10mm of tf] (bd) {\scriptsize R};

  \path[->]
    (t0) edge node[above, font=\scriptsize] {1} (t1)
    (t1) edge node[above, font=\scriptsize] {0,1} (t2)
    (t2) edge node[above, font=\scriptsize] {1} (t3)
    (t3) edge node[above, font=\scriptsize] {0,1} (t4)
    (t4) edge node[above, font=\scriptsize] {0,1} (tf);

  \path[->]
    (t0) edge[bend right=18] node[left, font=\scriptsize] {0} (b1)
    (t2) edge[bend right=18] node[left, font=\scriptsize] {0} (b3);

  \path[->]
    (b1) edge node[above, font=\scriptsize] {0,1} (b2)
    (b2) edge node[above, font=\scriptsize] {0,1} (b3)
    (b3) edge node[above, font=\scriptsize] {0,1} (b4)
    (b4) edge[bend left=10]  node[right, font=\scriptsize] {1} (tf)
    (b4) edge[bend right=14] node[right, font=\scriptsize] {0} (bd)
    (bd) edge[loop below] node[font=\scriptsize] {0,1} (bd)
    (tf) edge[bend left=14] node[right, font=\scriptsize] {0,1} (bd);
\end{tikzpicture}%
}
\caption{Transformed DFA $(z_1\wedge z_3)\vee z_5$, $x\in\{0,1\}^5$.}
\end{subfigure}

\caption{Transformation from Section~\ref{subsec:bdfa-construct}.  
In (b), green states are from the original DFA; blue states $q'_1,\ldots,q'_4$ ensure every prefix of length $<N$ has continuation labels $[1,1,0]$, and $q'_4$ routes to accept on input $1$ (to the dead state on $0$).}
\label{fig:hard-construct}
\end{figure}

\begin{lemma}[Padded Construction]\label{lem:padding-construct}
From a Boolean Acyclic DFA $A$ we can construct, in time polynomial in $|A|+N$, a DFA $A^{\oplus}$ with $L_{A^{\oplus}}\subseteq\{0,1\}^{N+1}$ and
\begin{equation}\label{eq:pad-language}
\text{for } u\in\{0,1\}^{N}\ \text{and}\ b\in\{0,1\}:\quad
u\cdot b\in L_{A^{\oplus}}
\ \Longleftrightarrow\
\big(A(u)=1\big)\,\ \text{or} \,\ \big(b=1\big).
\end{equation}
Moreover, for every prefix $y$ with $|y|<N$,
\[
\big(\dfunc(y,0),\ \dfunc(y,1),\ L_{A^{\oplus}}(y)\big)=(1,1,0),
\]
and for every $u\in\{0,1\}^{N}$,
\[
\big(\dfunc(u,0),\ \dfunc(u,1),\ L_{A^{\oplus}}(u)\big)=\big(A(u),\,1,\,0\big).
\]
The construction adds at most $N{+}1$ states.
\end{lemma}

\begin{proof}
Minimize the DFA $A$ if it is not minimal. Create new states $q'_1,\ldots,q'_N$ and a new accepting state $q_{\mathrm{fin}}$.  For $1\le i<N$ set
\[
\delta_{A^{\oplus}}(q'_i,0)=q'_{i+1},\qquad \delta_{A^{\oplus}}(q'_i,1)=q'_{i+1},
\]
and at $i=N$ set
\[
\delta_{A^{\oplus}}(q'_N,1)=q_{\mathrm{fin}},\qquad \delta_{A^{\oplus}}(q'_N,0)=\deadstate.
\]
From $q_{\mathrm{fin}}$ send both symbols to $\deadstate$ (so acceptance occurs only at length $N{+}1$).

Now modify $A$ to create $A^{\oplus}$ as follows.  For each non-dead state $q$ with $\ell(q)<N$ and each $\sym\in\{0,1\}$:
\begin{itemize}
\item If $\delta_{A}(q,\sym)=\deadstate$ in $A$, set $\delta_{A^{\oplus}}(q,\sym)=q'_{\ell(q)+1}$ (redirect the dead transition into the chain).
\item Otherwise set $\delta_{A^{\oplus}}(q,\sym)=\delta_{A}(q,\sym)$.
\end{itemize}
For each state $q$ at depth $N$ set
\[
\delta_{A^{\oplus}}(q,1)=q_{\mathrm{fin}},\qquad
\delta_{A^{\oplus}}(q,0)=\begin{cases}
q_{\mathrm{fin}},& q\in \final_{A},\\
\deadstate,& q\notin \final_{A}.
\end{cases}
\]

See Figure~\ref{fig:hard-construct} for a simple example construction for a Boolean Acyclic DFA computing a conjunction.

All transitions out of $\deadstate$ point to $\deadstate$, i.e., $\delta_{A^{\oplus}}(\deadstate,\sym)=\deadstate$ for both symbols.  (If $A$ recognizes the empty language, then $Q=\{\deadstate\}$; in this case we additionally introduce a fresh start state $\tilde q_0$ of depth $0$ with $\delta_{A^{\oplus}}(\tilde q_0,0)=\delta_{A^{\oplus}}(\tilde q_0,1)=q'_1$ and take $\tilde q_0$ as the start state; the conclusions below still hold.)

By construction, acceptance can occur only at $q_{\mathrm{fin}}$ after exactly $N{+}1$ symbols, so $L_{A^{\oplus}}\subseteq\{0,1\}^{N+1}$.  The rule at depth $N$ gives~\eqref{eq:pad-language}.  For any prefix $y$ with $|y|<N$, either an original transition still has an accepting continuation, or a dead transition is redirected to the chain $q'_{|y|+1}\!\to\cdots\to q'_N\!\to q_{\mathrm{fin}}$ (taking the last symbol $1$). Hence $\dfunc(y,0)=\dfunc(y,1)=1$ and $L_{A^{\oplus}}(y)=0$ for any $y$ with $|y| < N$.  At depth $N$, for every $u\in\{0,1\}^{N}$ our construction enforces
\[
\delta_{A^{\oplus}}\!\big(\,\delta_{A^{\oplus}}(q_0,u),1\big)=q_{\mathrm{fin}}\quad\text{and}\quad
\delta_{A^{\oplus}}\!\big(\,\delta_{A^{\oplus}}(q_0,u),0\big)=
\begin{cases}
q_{\mathrm{fin}},& \delta_{A}(q_0,u)\in \final_{A},\\[2pt]
\deadstate,& \delta_{A}(q_0,u)\notin \final_{A},
\end{cases}
\]
so $\dfunc(u,1)=1$, $\dfunc(u,0)=A(u)$, and $L_{A^{\oplus}}(u)=0$, which is exactly
\[
\big(\dfunc(u,0),\ \dfunc(u,1),\ L_{A^{\oplus}}(u)\big)=\big(A(u),\,1,\,0\big).
\]
\end{proof}

\textbf{Reduction from standard learning to $\nsp$ learning.}
Let $D$ be any distribution on $\{0,1\}^{N}$.  Define the \emph{padded} distribution $D^{\oplus}$ on $\{0,1\}^{N+1}$ by sampling $u\sim D$ and returning $x:=u\cdot 1$.  By \eqref{eq:pad-language}, $x\in L_{A^{\oplus}}$ for every $u$, so $D^{\oplus}$ is supported on positive examples as required by the $\nsp$ setting.  Moreover, given a labeled standard example $(u,y)$ with $y=A(u)$, the full $\nsp$ label vector $f_{A^{\oplus}}(x)$ for $x=u\cdot 1$ is computable from $(u,y)$:
\begin{align*}
&\text{for } 0\le \ell<N:\quad \big(\dfunc(\prefl{\ell},0),\dfunc(\prefl{\ell},1),L(\prefl{\ell})\big)=(1,1,0),\\
&\text{for } \ell=N:\quad \big(\dfunc(\prefl{N},0),\dfunc(\prefl{N},1),L(\prefl{N})\big)=(y,1,0),\\
&\text{for } \ell=N{+}1:\quad \big(\dfunc(\prefl{N+1},0),\dfunc(\prefl{N+1},1),L(\prefl{N+1})\big)=(0,0,1).
\end{align*}
Here $\prefl{\ell}$ denotes the length-$\ell$ prefix of the padded string $x$.

\begin{theorem}\label{th:NSPtoSTD}
Fix $N$ and a polynomial $p(\cdot)$. If $\mathrm{ADFA}^{N}_{p(N)}$ is efficiently PAC-learnable in the $\nsp$ setting from positive examples, then $\mathrm{ADFA}^{N}_{p(N)}$ is efficiently PAC-learnable in the conventional classification (binary-label) setting.    
\end{theorem}

\begin{proof}
Let $\A_{\nsp}$ be an efficient $\nsp$ learner for $\mathrm{ADFA}^{N}_{p(N)}$.  From i.i.d.\ labeled samples $(u^{(i)},y^{(i)})$ with $y^{(i)}=A(u^{(i)})$, form positive $\nsp$ examples $\big(x^{(i)}, f_{A^{\oplus}}(x^{(i)})\big)$ where $x^{(i)}=u^{(i)}\!\cdot 1$ using the rule above, and feed them to $\A_{\nsp}$.  Let $\fhat$ be the returned predictor so that, with probability at least $1-\delta$,
\[
\errnsp\big(\fhat;\, f_{A^{\oplus}},\, D^{\oplus}\big)
=\E_{u\sim D}\Big[\mathbb{I} \big[\fhat(u\!\cdot 1)\neq f_{A^{\oplus}}(u\!\cdot 1)\big ]\Big]\ \le\ \epsilon.
\]
Define a standard classifier $h:\{0,1\}^{N}\to\{0,1\}$ by
\[
h(u)\ :=\ \text{the bit predicted by }\fhat\text{ for }\dfunc(\prefl{N},0)\text{ on }x=u\!\cdot 1.
\]
Whenever $\mathbb{I}\big [\fhat(x)=f_{A^{\oplus}}(x)\big ]=1$, Lemma~\ref{lem:padding-construct} gives $h(u)=A(u)$.  Therefore
\[
\Pr_{u\sim D}\big[h(u)\neq A(u)\big]\ \le\ \Pr_{u\sim D}\Big[\fhat(u\!\cdot 1)\neq f_{A^{\oplus}}(u\!\cdot 1)\Big]\ \le\ \epsilon,
\]
yielding an efficient PAC learner in the standard setting.  The state complexity of $A^{\oplus}$ is at most $p(N)+N+1$, preserving polynomial time complexity.
\end{proof}

\begin{corollary}[Cryptographic hardness for $\nsp$ learning of fixed-length acyclic DFAs]\label{cor:nsphard-bdfa}
Under the assumptions of Theorem~\ref{th:kearns}, there is no polynomial-time weak learner for $\mathrm{ADFA}_{p(\cdot)}$ in the $\nsp$ setting.  Otherwise, Theorem~\ref{th:NSPtoSTD} would yield a polynomial-time weak learner in the standard setting, contradicting Theorem~\ref{th:kearns}.
\end{corollary}


\textbf{Boolean formulas.} Exactly the same padding idea applies to Boolean formulas.  Let $F$ be any class of formulas $f:\{0,1\}^{N}\to\{0,1\}$ and define
\[
f'(z_1,\ldots,z_{N+1})\ :=\ f(z_1,\ldots,z_N)\ \vee\ z_{N+1}.
\]
Given a labeled standard example $(u,y)$ with $y=f(u)$, set $x:=u\cdot 1$.  The $\nsp$ labels for the positive string $x$ under $f'$ are then computable from $(u,y)$:
\[
\text{for }\ell<N:\ (1,1,0),\qquad \text{for }\ell=N:\ (y,1,0),\qquad \text{for }\ell=N{+}1:\ (0,0,1),
\]
with the same ordering (continuations first, then membership) as in \S\ref{sec:def}.  Thus, an efficient $\nsp$ learner for $F':=\{f':f\in F\}$ yields, by reading the depth-$N$ continuation bit for symbol $0$, an efficient PAC learner for $F$ in the standard setting.  In particular, cryptographic hardness results for learning formulas (e.g., via $\mathrm{NC}^1$) carry over to the $\nsp$ setting by this reduction.

\subsection{Discussion} 

The result shows that richer supervision via $\nsp$ does not circumvent the computational barrier for learning regular languages: under standard assumptions, there is no polynomial-time weak learner for $\BDFA_{p(\cdot)}$ and hence for DFAs even in the $\nsp$ setting. This remains true when the learner receives both positive and negative examples with $\nsp$ labels, showing that such additional supervision does not mitigate these barriers. The hardness is \emph{improper}: it rules out efficient learning even when the hypothesis need not be a DFA, and thus applies to neural models trained to match $\nsp$ labels.

\bibliographystyle{plainnat}
\bibliography{citations}

\end{document}